\documentclass[a4paper,UKenglish,cleveref,autoref,thm-restate,pdfa]{lipics-v2021}
\nolinenumbers
\pdfoutput=1 %uncomment to ensure pdflatex processing (mandatatory e.g. to submit to arXiv)
\hideLIPIcs  %uncomment to remove references to LIPIcs series (logo, DOI, ...), e.g. when preparing a pre-final version to be uploaded to arXiv or another public repository
\newcommand{\m}[1]{[\![#1]\!]}
\newcommand{\ms}[1]{[\![#1]\!]}
\newcommand{\pms}[1]{[\![\![#1]\!]\!]}
\newcommand{\ams}[1]{(\!(#1)\!)}
\newcommand{\pams}[1]{(\!(\!(#1)\!)\!)}

\newcommand{\cent}{\mathrel{\scalebox{1}[1.5]{$\shortmid$}\mkern-3.1mu\raisebox{0.1ex}{$=$}}}
\newcommand{\ent}{\mathrel{\scalebox{1}[1.5]{$\shortmid$}\mkern-3.1mu\raisebox{0.1ex}{$\equiv$}}}
\DeclareMathOperator*{\argmax}{arg\,max}

\title{Inference of Abstraction for Grounded Predicate Logic} %TODO Please add
\author{Hiroyuki Kido}{Cardiff University, Park Place, Cardiff, CF10 3AT, United Kingdom}{KidoH@cardiff.ac.uk}{https://orcid.org/0000-0002-7622-4428}{}

\authorrunning{Hiroyuki Kido}

\Copyright{Hiroyuki Kido} 

\ccsdesc[100]{Computing methodologies → Artificial intelligence → Knowledge representation and reasoning → Probabilistic reasoning}

\keywords{Abstraction, Symbol grounding, Reasoning and learning, Generative models, Knowledge acquisition bottleneck, Top-down and bottom-up processing, Bayesian brain} %TODO mandatory; please add comma-separated list of keywords

\EventEditors{John Q. Open and Joan R. Access}
\EventNoEds{2}
\EventLongTitle{42nd Conference on Very Important Topics (CVIT 2016)}
\EventShortTitle{CVIT 2016}
\EventAcronym{CVIT}
\EventYear{2016}
\EventDate{December 24--27, 2016}
\EventLocation{Little Whinging, United Kingdom}
\EventLogo{}
\SeriesVolume{42}
\ArticleNo{23}
\begin{document}

\maketitle

\begin{abstract}
An important open question in AI is what simple and natural principle enables a machine to reason logically for meaningful abstraction with grounded symbols. This paper explores a conceptually new approach to combining probabilistic reasoning and predicative symbolic reasoning over data. We return to the era of reasoning with a full joint distribution before the advent of Bayesian networks. We then discuss that a full joint distribution over models of exponential size in propositional logic and of infinite size in predicate logic should be simply derived from a full joint distribution over data of linear size. We show that the same process is not only enough to generalise the logical consequence relation of predicate logic but also to provide a new perspective to rethink well-known limitations such as the undecidability of predicate logic, the symbol grounding problem and the principle of explosion. The reproducibility of this theoretical work is fully demonstrated by the included proofs.
\end{abstract}
%%%%%%%%%%%%%%%%%%%%%%%%%%%%%%%%%%%%%%%%%%%%%%%%%%%%%%%%%%%%%%%%%%%%%%%%%%%%%%%%%%%%%%%%%%%%%%%%%%%%%%%%%%%%%%%%%%%%%%%%%%%%%%%%%%%%%%%%%%%%%%%%%%%%%%%%%%%%%%%%%%%%%%%%%%%%%%%%%%%%%%%%%%%%%%%%%%%%%%%%%%%%%%%%%%%%%%%%%%%%%%%%%%%%%%%%%%%%%
\section{Introduction}\label{sec1}
The current artificial intelligence (AI) systems such as large language models (LLMs) \cite{ChatGPT,DeepSeek} demonstrate a surprising linguistic ability in both what they know and how they articulate it. However, the common view is that they are still not as capable as ordinary people in several areas such as logical reasoning and abstract reasoning. For logical reasoning, it is unlikely to believe that the statistical patterns an AI algorithm extracts from finite training data can capture the infinite set of rules of valid inference studied in formal logic. For abstract reasoning, it is still unclear how a machine should explore and discover abstract concepts and principles from the real world in its own way. Consider the following problems requiring both abstract reasoning and logical reasoning skills.
\begin{example}\label{sec1:ex1}
Carol remembers the following three scenes.
\begin{itemize}
\item Alice and Bob did not blame each other.
\item One day Alice blamed Bob, and she blamed herself afterwards.
\item Alice and Bob blamed each other on another day.
\end{itemize}
One day Carol wants to blame Bob. She hesitated it because she learnt that someone will blame those who blame anyone, which can be expressed in a predicate language as follows.
\begin{eqnarray*}
\forall x\exists y~Blames(x,y)\to\exists z~Blames(z,x)
\end{eqnarray*}
\end{example}
\begin{example}\label{sec1:ex2}
Consider the following three data. What number fits in the blank?
\begin{center}
\includegraphics[scale=0.3]{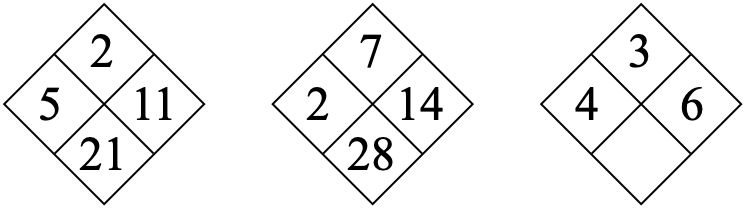}
\end{center}
The correct number could be 18 as the following predicate knowledge can be extracted.
\begin{eqnarray*}
\it{top}\times\it{left}+\it{right}=\it{bottom}
\end{eqnarray*}
\end{example}
Interestingly, the current AI systems such as ChatGPT \cite{ChatGPT} and DeepSeek \cite{DeepSeek} often answer them incorrectly. Their failure is mainly due to a lack of abstract or logical reasoning skills, but not a lack of arithmetic skills. Abstraction and Reasoning Corpus (ARC) \cite{Chollet:19} is a benchmark test designed to test machine's ability to extract graphical patterns from images, but not the intellectual ones shown above.
\par
In this paper, we ask how logical reasoning in predicate logic emerges from reasoning based on data. The underlying idea discussed in this paper is abstraction. Roughly speaking, it is about an inferential process of deriving intrinsically abstract symbols from intrinsically concrete data through selective ignorance. It is not about generalisation where typical inferential processes, e.g., deductive reasoning, is used backward for general rules from specific examples or facts. This type of reasoning is intensively studied as inverse resolution \cite{Muggleton:88,Nienhuys:97}, inverse deduction \cite{Russell:20} and inverse entailment \cite{Muggleton:95} mainly in inductive logic programming (ILP) \cite{Nienhuys:97}. It is either not about parametric learning where intrinsically concrete data are assumed to be generated from parameters of a probability distribution. This idea is prevalent in various applications of machine learning and statistics, e.g., \cite{Bishop:06,Tenenbaum:06,Dasgupta:20,Lake:15,Lake:17}. Abstraction is rather relevant to top-down (memory/experience-driven) and bottom-up (sensory-driven) information processing used by neuroscientists and AI researchers as a metaphor for the cognitive process of biological brains, e.g., \cite{pearl:03,Harnad:90,lee:03,Hawkins:21,Gregory:97,Rao:99,friston:10}.
\par
In this paper, we extend the inference of propositional abstraction \cite{kido:24-1,kido:24-2} to the inference of predicative abstraction towards enhanced human-like machine intelligence. The key idea is to use the property of predicate logic and expand the joint probability distribution over data, models of predicate logic and predicate formula, denoted by $D$, $M$ and $\alpha$, respectively, as follows.
\begin{eqnarray*}
&p(D,M,\alpha)=p(\alpha|M)p(M|D)p(D).
\end{eqnarray*}
The right-hand side realises the idea that a formula is an abstraction, or selective ignorance, of models and each of the models is an abstraction of data.
\par
We show that the inference of predicative abstraction serves as a solution to simple yet important problems such as Examples \ref{sec1:ex1} and \ref{sec1:ex2}. The research is not as straightforward as we think because the semantics of predicate logic needs a reformulation in accordance with abstraction. Our theory assumes only closed formulas, i.e., predicate formulas without free variables, to balance the expressiveness and simplicity of the theory. The contributions of this paper are summarised as follows.
\begin{itemize}
\item We introduce a simple theory of inference that opens up the possibility of combining probability theory and predicate logic in a data-driven manner. Predicate reasoning in our theory always proceeds between data and predicate formulas. This suggests a shift in the traditional view that predicate reasoning proceeds between predicate formulas via rules of inference (see Section \ref{sec2}).
\item The theory allows us to see the traditional model-based predicate reasoning as a special case of data-based predicate reasoning studied in this paper. The data-based perspective provides a new opportunity to rethink some existing limitations such as the undecidability of predicate logic, the symbol grounding problem \cite{Harnad:90,Russell:20} and commonsense reasoning \cite{Brewka:91,Davis:15} (see Sections \ref{sec31}, \ref{sec32} and \ref{sec33}).
\item We demonstrate a solution to simple yet essential problems that are often difficult to solve by existing established approaches (see Section \ref{sec34}).
\end{itemize}
%%%%%%%%%%%%%%%%%%%%%%%%%%%%%%%%%%%%%%%%%%%%%%%%%%%%%%%%%%%%%%%%%%%%%%%%%%%%%%%%%%%%%%%%%%%%%%%%%%%%%%%%%%%%%%%%%%%%%%%%%%%%%%%%%%%%%%%%%%%%%%%%%%%%%%%%%%%%%%%%%%%%%%%%%%%%%%%%%%%%%%%%%%%%%%%%%%%%%%%%%%%%%%%%%%%%%%%%%%%%%%%%%%%%%%%%%%%%%%%%%%%%%%%%%%%%%%%%%%%%%%%%%%%%%%%%%%%%%%%%%%%%%%%%%%%%%%%%%%%%%%%%%%%%%%%%%%%%%%%%%%%%%%%%%%%%%%%%%%%%%%%%%%%%%%%%%%%%%%%%%%%%%%%%%%%%%%%%%%%%%%%%%%%%%%%%%%%%%%%%%%%%%%%%%%%%%%%%%%%%%%%%%%%%%%%%%%%%%%%%%%%%%%%%%%%%%%%%%%%%%%%%%%%%%%%%%%%%%%%%%%%%%%%%%%%%%%%%%%%%%%%%%%%%%%%%%%%%%%%%%%%%%%%%%%%%
\section{Proposals}\label{sec2}
\subsection{Data support models}\label{sec21}
\begin{figure*}[t]
  \begin{center}
   \includegraphics[scale=0.28]{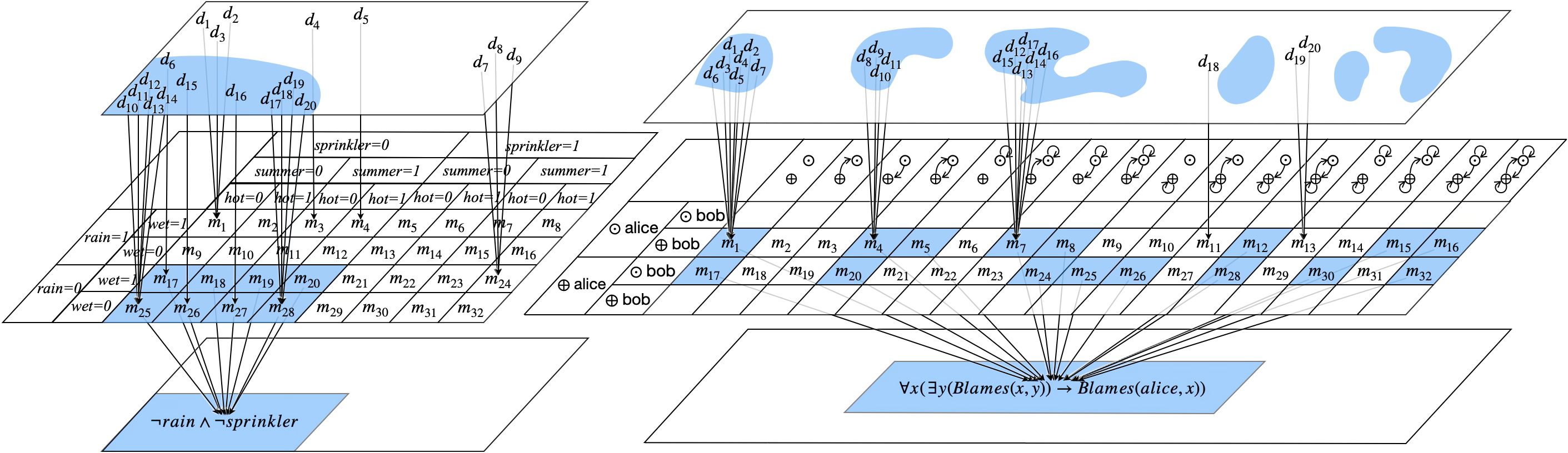}
  \end{center}
\caption{The hierarchy shown on the left is our illustration of the existing work on the inference of propositional abstraction \cite{kido:24-1,kido:24-2}. The one shown on the right is an illustration of our work on the inference of predicative abstraction. The top layers are both distributions of data. The middle layer on the left is a distribution of models in propositional logic, i.e., valuations. The one on the right is a distribution of models in predicate logic, i.e., pairs of domains of discourse and valuation functions. The bottom layer on the left is a distribution of the truth values of the propositional formula, whereas the one on the right is the same type of distribution for the predicate formula.}
\label{sec3:fig1}
\end{figure*}
The inference of abstraction for propositional logic \cite{kido:24-1,kido:24-2} is insufficient to handle problems like Examples \ref{sec1:ex1} and \ref{sec1:ex2}. We thus propose the inference of abstraction for predicate logic in this section.
Let $\{d_{1},d_{2},...,d_{K}\}$ be a multiset of $K$ data and $D$ be a random variable for data taking values from $\{d_{1},d_{2},...,d_{K}\}$.
\begin{definition}\label{sec3:def01}
Let $d_{k}\in\{d_{1},d_{2},...,d_{K}\}$. Th probability of $d_{k}$, denoted by $p(D=d_{k})$, is defined as follows.
\begin{align}
p(D=d_{k})=\frac{1}{K}\label{sec3:eq1}
\end{align}
\end{definition}
Namely, $p(D)$ is a uniform distribution. Let ${\cal C}$, ${\cal V}$, ${\cal F}$ and ${\cal P}$ be the sets of constants, variables, function symbols and predicate symbols, respectively, and $L$ be the predicate language built with these vocabularies.
\begin{example}\label{sec3:ex1}
Consider the following vocabularies of a predicate language.
\begin{itemize}
\item Constants: ${\cal C}=\{alice, bob\}$
\item Variables: ${\cal V}=\{x, y\}$
\item Function symbols: ${\cal F}=\{mentor\}$
\item Predicate symbols: ${\cal P}=\{Blames\}$
\end{itemize}
The following is a predicate formula meaning that Alice's mentor blames everyone who blames someone.
\begin{eqnarray*}
\forall x(\exists y(Blames(x,y))\to Blames(mentor(alice),x))
\end{eqnarray*}
\end{example}
In this paper, we assume that the predicate language includes only formulas without free variables, i.e., closed formulas.\footnote{The truth values of closed formulas depend only on a model, whereas the truth values of open formulas, i.e., formulas with free variables, depend additionally on an assignment, a function mapping each variable to an entity in the domain of discourse.} We exclude open formulas for the following reasons. First, it is inappropriate to view an assignment in predicate logic as an abstraction, or selective ignorance, of data or observations. Its inclusion thus does not fit the underlying idea of the inference of abstraction. Second, a lot of cases such as Examples \ref{sec1:ex1} and \ref{sec1:ex2} do not need open formulas. Its inclusion thus makes our formalism unnecessary complicated.
\par
As usual, a model in predicate logic is a pair of a domain of discourse and valuation function. The domain of discourse, denoted by $u$, is a non-empty set of a finite or countably infinite number of entities. The valuation function, denoted by $v$, is a function that associates constants, function symbols and predicate symbols with $u$. We use the symbol $ar(x)$ to denote the arity of the function or predicate symbol $x$. Specifically,
\begin{itemize}
\item $v$ maps each constant $c\in{\cal C}$ to an entity of $u$, i.e., $v(c)\in u$.\footnote{\label{sec3:footnote1}We assume that $v$ is surjective with respect to ${\cal C}$, meaning that there is constant $c\in{\cal C}$ such that $v(c)=e$, for all entities $e\in u$. This assumption allows us to apply simple semantics of predicate logic.}
\item $v$ maps each function symbol $f\in{\cal F}$ to an $ar(f)$-ary function from $u$ to $u^{ar(f)}$, i.e., $v(f):u^{ar(f)}\to u$.
\item $v$ maps each predicate symbol $P\in{\cal P}$ to a subset of $u^{ar(P)}$, i.e., $v(P)\subseteq u^{ar(P)}$.
\end{itemize}
From the viewpoint of the inference of abstraction, it is important to adopt the perspective that each model represents a different state of the world. For any function symbol $f\in{\cal F}$, we write $v(f)(v(t_{1}),...,v(t_{ar(f)}))$ as $v(f(t_{1},...,t_{ar(f)}))$, where $t_{1}, t_{2},... t_{ar(f)}$ are the arguments of $f$ called terms referring to constants, variables or functions.
\par
Let $\{m_{1}, m_{2},...,m_{N}\}$ be the set of models of the predicate language $L$. This set is finite or countably infinite. Let $M$ be a random variable for the models taking values from $\{m_{1}, m_{2},...,m_{N}\}$. We assume that each data point supports a single model. We thus assume a function $m: \{d_{1},d_{2},...,d_{K}\}\to\{m_{1}, m_{2},...,m_{N}\}$ such that $m(d_{k})$ denotes the model supported by data $d_{k}$.
\begin{definition}\label{sec3:def02}
Let $d_{k}\in\{d_{1},d_{2},...,d_{K}\}$ and $m_{n}\in\{m_{1}, m_{2},...,m_{N}\}$. The probability of $m_{n}$ given $d_{k}$, denoted by $p(M=m_{n}|D=d_{k})$, is defined as follows.
\begin{eqnarray}
p(M=m_{n}|D=d_{k})=
\begin{cases}
1 & \text{if } m_{n}=m(d_{k})\\
0 & \text{otherwise}
\end{cases}\label{sec3:eq2}
\end{eqnarray}
\end{definition}
We use the symbols $u(m_{n})$ and $v(m_{n})$ to denote the domain of discourse and the valuation function of the model $m_{n}$, i.e., $m_{n}=\langle u(m_{n}), v(m_{n})\rangle$. Thus, the model supported by data $d_{k}$ can be written as $m(d_{k})=\langle u(m(d_{k})), v(m(d_{k}))\rangle$.
\begin{example}\label{sec3:ex3}
Consider the predicate language $L$ built with the following vocabularies.
\begin{itemize}
\item Constants: ${\cal C}=\{alice, bob\}$
\item Variables: ${\cal V}=\{x, y\}$
\item Function symbols: ${\cal F}=\emptyset$
\item Predicate symbols: ${\cal P}=\{Blames\}$
\end{itemize}
The top layer of the hierarchy shown on the right in Figure \ref{sec3:fig1} shows twenty data. Given $u=\{\odot,\oplus\}$, the middle layer shows all the thirty two models $\langle u,v\rangle$ of the language $L$. The depth of the middle layer shows how the valuation function $v$ associates the constants $alice$ and $bob$ with $u$. Its width shows how $v$ associates the predicate symbol $Blames$ to $u$, where an arrow from $x$ to $y$ represents that $x$ blames $y$. Each blank cell in the middle layer is not a model due to the assumption we made in Footnote \ref{sec3:footnote1}. The arrow from the top to middle layers represent a function $m$. The twenty data $d_{k}$ commonly say that there are two people $\odot$ and $\oplus$ named Alice and Bob, respectively , i.e., $u(m(d_{k}))=\{\odot,\oplus\}$, $v(m(d_{k}))(alice)=\odot$ and $v(m(d_{k}))(bob)=\oplus$. Each data point, however, supports a different situation. Specifically, the functions $m$ and $v$ are given as follows, for all data $d_{k}$.
\begin{eqnarray*}
&m(d_{k})=
\begin{cases}
m_{1}&\text{if }k\in\{1\text{-}7\}\\
m_{4}&\text{if }k\in\{8\text{-}11\}\\
m_{7}&\text{if }k\in\{12\text{-}17\}\\
m_{11}&\text{if }k\in\{18\}\\
m_{13}&\text{if }k\in\{19,20\}
\end{cases}
&v(m(d_{k}))(Blames)=
\begin{cases}
\emptyset&\text{if }k\in\{1\text{-}7\}\\
\{(\odot,\oplus),(\oplus,\odot)\}&\text{if }k\in\{8\text{-}11\}\\
\{(\odot,\odot),(\odot,\oplus)\}&\text{if }k\in\{12\text{-}17\}\\
\{(\odot,\oplus),(\oplus,\oplus)\}&\text{if }k\in\{18\}\\
\{(\odot,\odot),(\oplus,\oplus)\}&\text{if }k\in\{19,20\}
\end{cases}
\end{eqnarray*}
\end{example}
%%%%%%%%%%%%%%%%%%%%%%%%%%%%%%%%%%%%%%%%%%%%%%%%%%%%%%%%%%%%%%%%%%%%
%%%%%%%%%%%%%%%%%%%%%%%%%%%%%%%%%%%%%%%%%%%%%%%%%%%%%%%%%%%%%%%%%%%%
%%%%%%%%%%%%%%%%%%%%%%%%%%%%%%%%%%%%%%%%%%%%%%%%%%%%%%%%%%%%%%%%%%%%
\subsection{Models support formulas}\label{sec22}
We are interested in the probability of predicate formula $\alpha\in L$ being true or false. We thus assume that each formula is a random variable taking values from $\{0,1\}$. For any truth values $v\in\{0,1\}$, we use the symbol $\ms{\alpha=v}$ to denote the set of models where $\alpha$ has the truth value $v$. We often write $\ms{\alpha=v}_{m_{n}}=1$ if $m_{n}\in\ms{\alpha=v}$ and $\ms{\alpha=v}_{m_{n}}=0$ otherwise for the membership of the model. We call formulas with neither logical connectives, such as $\lnot$, $\lor$, $\land$ and $\to$, nor quantifiers, such as $\forall$ and $\exists$, atomic formulas. Let $m_{n}=\langle u,v\rangle$ be a model. As usual, the truth value of an atomic formula without variables is defined as follows.
\begin{align*}
&\ms{P(t_{1},...,t_{ar(P)})}_{m_{n}}=
\begin{cases}
1&\text{if } (v(t_{1}),...,v(t_{ar(P)}))\in v(P)\\
0&\text{otherwise}
\end{cases}
\end{align*}
Let $\alpha, \beta\in L$ be formulas and $m_{n}=\langle u,v\rangle$ be a model. As usual, the truth values of compound formulas with logical connectives are defined as follows.
\begin{eqnarray*}
&&\ms{\lnot\alpha}_{m_{n}}=1 \Leftrightarrow \ms{\alpha}_{m_{n}}=0\\
&&\ms{\alpha\land\beta}_{m_{n}}=1 \Leftrightarrow \ms{\alpha}_{m_{n}}=1 \text{ and }\ms{\beta}_{m_{n}}=1 \Leftrightarrow \min\{\ms{\alpha}_{m_{n}},\ms{\beta}_{m_{n}}\}\\
&&\ms{\alpha\lor\beta}_{m_{n}}=1 \Leftrightarrow \ms{\alpha}_{m_{n}}=1 \text{ or }\ms{\beta}_{m_{n}}=1 \Leftrightarrow \max\{\ms{\alpha}_{m_{n}},\ms{\beta}_{m_{n}}\}\\
&&\ms{\alpha\to\beta}_{m_{n}}=1 \Leftrightarrow \ms{\alpha}_{m_{n}}=0 \text{ or }\ms{\beta}_{m_{n}}=1 \Leftrightarrow \max\{1-\ms{\alpha}_{m_{n}},\ms{\beta}_{m_{n}}\}
\end{eqnarray*}
Let us use the symbol $\alpha[c/x]$ to denote the formula replacing all the free variables $x\in{\cal V}$ in the formula $\alpha$ by the constant $c\in{\cal C}$. Here, a variable $x$ is free if there is no quantifier bounding $x$ or $x$ is outside the scope of such quantifiers. As usual, the truth values of compound formulas with quantifiers are defined as follows.\footnote{This definition is based on the assumption we made in Footnote \ref{sec3:footnote1}. With this assumption, we can define the truth solely in terms of constants, without referring to the domain of discourse.}
\begin{eqnarray*}
&&\ms{\forall x~\alpha}_{m_{n}}=1 \Leftrightarrow \ms{\alpha[c/x]}_{m_{n}}=1\text{, for all }c\in{\cal C} \Leftrightarrow \min_{c\in{\cal C}}\{\ms{\alpha[c/x]}_{m_{n}}\}\\
&&\ms{\exists x~\alpha}_{m_{n}}=1 \Leftrightarrow \ms{\alpha[c/x]}_{m_{n}}=1\text{, for some }c\in{\cal C} \Leftrightarrow \max_{c\in{\cal C}}\{\ms{\alpha[c/x]}_{m_{n}}\}
\end{eqnarray*}
We say that a formula $\alpha\in L$ is true in a model $m_{n}$ or $m_{n}$ satisfies, or supports, $\alpha$ if $\ms{\alpha}_{m_{n}}=1$. We also write $\ms{\alpha}_{m_{n}}=1$ as $m_{n}\in\ms{\alpha}$.
\begin{example}[Continued]\label{sec3:ex5}
Let us find the models where everyone blames someone, i.e., $\forall x\exists y$ $Blames(x,y)$, is true. Each atomic formula has the following truth value (see Figure \ref{sec3:fig1}).
\begin{eqnarray*}
&&\ms{Blames(alice,alice)}=\{m_{n}|n\in\{5\text{-}8,13\text{-}16,25\text{-}32\}\}\\
&&\ms{Blames(alice,bob)}=\{m_{n}|n\in\{3,4,7,8,11,12,15,16,18,20,22,24,26,28,30,32\}\}\\
&&\ms{Blames(bob,alice)}=\{m_{n}|n\in\{2,4,6,8,10,12,14,16,19,20,23,24,27,28,31,32\}\}\\
&&\ms{Blames(bob,bob)}=\{m_{n}|n\in\{9\text{-}16,21\text{-}24,29\text{-}32\}\}
\end{eqnarray*}
The compound formula, $\forall x\exists y$ $Blames(x,y)$, thus has the following truth value.
\begin{align*}
&\ms{\forall x\exists y~Blames(x,y)}_{m_{n}}=\min_{c_{1}\in{\cal C}}\big\{\ms{\exists y~Blames(c_{1},y)}_{m_{n}}\big\}\\
&=\min_{c_{1}\in{\cal C}}\big\{\max_{c_{2}\in{\cal C}}\big\{\ms{Blames(c_{1},c_{2})}_{m_{n}}\big\}\big\}\\
&=\min\left\{\begin{array}{l}
\max\big\{\ms{Blames(alice,alice)}_{m_{n}}, \ms{Blames(alice,bob)}_{m_{n}}\big\},\\
\max\big\{\ms{Blames(bob,alice)}_{m_{n}}, \ms{Blames(bob,bob)}_{m_{n}}\big\}
\end{array}\right\}\\
&=
\begin{cases}
1&\text{if }n\in\{3\text{-}8,11\text{-}16,18,20,22,24\text{-}32\}\cap\{2,4,6,8\text{-}16,19\text{-}24,27\text{-}32\}\text{, i.e.,}\\
&\text{if }n\in\{4,6,8,11\text{-}16,20,22,24,27\text{-}32\}\\
0&\text{otherwise}
\end{cases}
\end{align*}
\end{example}
The probability of the truth of a formula is defined using the semantics of predicate logic.
\begin{definition}\label{sec3:def03}
Let $\mu\in[0.5,1]$, $d_{k}\in\{d_{1},d_{2},...,d_{K}\}$, $m_{n}\in\{m_{1},m_{2},...,m_{N}\}$, $\alpha_{1},\alpha_{2},...,\alpha_{I}\in L$ and $v_{1},v_{2},...,v_{I}\in\{0,1\}$. The probability of $\alpha_{1}=v_{1}$ given $\alpha_{2}=v_{2},...,\alpha_{I}=v_{I}$, $M=m_{n}$ and  $D=d_{k}$, denoted by $p(\alpha_{1}=v_{1}|\alpha_{2}=v_{2},...,\alpha_{I}=v_{I}, M=m_{n}, D=d_{k})$, is defined as follows. 
\begin{eqnarray*}
&&p(\alpha_{1}=v_{1}|\alpha_{2}=v_{2},...,\alpha_{I}=v_{I}, M=m_{n}, D=d_{k})=
\begin{cases}
\mu & \text{if }m_{n}\in\ms{\alpha_{1}=v_{1}}\\
1-\mu & \text{otherwise}
\end{cases}
\end{eqnarray*}
\end{definition}
If we adopt the convention that $0^0=1$ then Definition \ref{sec3:def03} can be expressed as a Bernoulli distribution with parameter $\mu$.
\begin{align*}
p(\alpha_{1}=v_{1}|\alpha_{2}=v_{2},...,\alpha_{I}=v_{I}, M=m_{n}, D=d_{k})=\mu^{\ms{\alpha_{1}=v_{1}}_{m_{n}}}(1-\mu)^{1-\ms{\alpha_{1}=v_{1}}_{m_{n}}}
\end{align*}
In Figure \ref{sec3:fig1}, the arrows from the middle to the bottom layer of the hierarchy shown on the right indicate that the predicate formula is true in these models. The following probabilistic property of conditional independence comes directly from the property of predicate logic.
\begin{proposition}\label{sec2:prop1}
Let $\alpha_{1},\alpha_{2}\in L$. $\alpha_{1}$ is conditionally independent of $\alpha_{2}$ and $D$ given $M$, i.e., $p(\alpha_{1}|\alpha_{2},M,D)=p(\alpha_{1}|M)$.
\end{proposition}
\begin{proof}
Using the definition of conditional probability, the right-hand side can be written as 
\begin{align*}
&p(\alpha_{1}|M)=\frac{p(\alpha_{1},M)}{p(M)}.
\end{align*}
Using the sum rule \cite{Bishop:06} and the product rule \cite{Bishop:06} of probability theory, its numerator can be expanded as
\begin{align*}
p(\alpha_{1},M)&=\sum_{v_{2}}\sum_{d_{k}}p(\alpha_{1},\alpha_{2}=v_{2},M,D=d_{k})\\
&=\sum_{v_{2}}\sum_{d_{k}}p(\alpha_{1}|\alpha_{2}=v_{2},M,D=d_{k})p(\alpha_{2}=v_{2},M,D=d_{k}).
\end{align*}
Now, it is obvious from Definition \ref{sec3:def03} that neither $\alpha_{2}$ nor $D$ affects the value of $p(\alpha_{1}|\alpha_{2}=v_{2},M,D=d_{k})$. We can thus move it outward. Using the sum rule, we have
\begin{align*}
p(\alpha_{1},M)&=p(\alpha_{1}|\alpha_{2},M,D)\sum_{v_{2}}\sum_{d_{k}}p(\alpha_{2}=v_{2},M,D=d_{k})=p(\alpha_{1}|\alpha_{2},M,D)p(M).
\end{align*}
Taking into accout the denominator, the original expression can be written as
\begin{align*}
&p(\alpha_{1}|M)=\frac{p(\alpha_{1}|\alpha_{2},M,D)p(M)}{p(M)}=p(\alpha_{1}|\alpha_{2},M,D).
\end{align*}
\end{proof}
From the equation in Proposition \ref{sec2:prop1}, we can simplify Definition \ref{sec3:def03} as follows.
\begin{eqnarray}
p(\alpha_{1}=v_{1}|M=m_{n})=\mu^{\ms{\alpha_{1}=v_{1}}}(1-\mu)^{1-\ms{\alpha_{1}=v_{1}}}\label{sec3:eq3}
\end{eqnarray}
The following example shows why we need to assume $\mu\in[0.5,1]$, rather than simply $\mu=1$.
\begin{example}[Continued.]\label{sec2:ex2}
Let $\mu=1$ and $\alpha=Blames(alice,bob)\land\lnot Blames(alice,bob)$. Using the definition of conditional probability, the sum rule, the product rule and Proposition \ref{sec2:prop1}, we have
\begin{align*}
p(M=m_{4}|\alpha)&=\frac{p(M=m_{4},\alpha)}{p(\alpha)}=\frac{p(M=m_{4},\alpha)}{\sum_{n=1}^{32}p(M=m_{n},\alpha)}\\
&=\frac{p(\alpha|M=m_{4})p(M=m_{4})}{\sum_{n=1}^{32}p(\alpha|M=m_{n})p(M=m_{n})}=\frac{(1-\mu)p(M=m_{4})}{\sum_{n=1}^{32}(1-\mu)p(M=m_{n})}=\frac{0}{0}.
\end{align*}
Namely, the value is undefined due to division by zero. However, given $\mu\neq 1$ such as $\mu$ approaches one, denoted by $\mu\to 1$, the undefined value can be replaced by a reasonable one.
\begin{align*}
&p(M=m_{4}|\alpha)=\lim_{\mu\to 1}\frac{(1-\mu)p(M=m_{4})}{\sum_{n=1}^{32}(1-\mu)p(M=m_{n})}=\frac{p(M=m_{4})}{\sum_{n=1}^{32}p(M=m_{n})}=p(M=m_{4})
\end{align*}
Here, $\sum_{n=1}^{32}p(M=m_{n})=1$ as $p(M)$ is the probability distribution over all the models.
\end{example}
In Section \ref{sec3}, we will discuss that $\mu=1$ corresponds to the logical consequence relation and $\mu\to 1$ corresponds to its natural generalisations. In Figure \ref{sec3:fig1}, each arrow between the middle and bottom layers of the hierarchy shown on the right shows that the model satisfies or supports the formula, $\forall x(\exists y(Blames(x,y))\to Blames(alice,x))$.
%%%%%%%%%%%%%%%%%%%%%%%%%%%%%%%%%%%%%%%%%%%%%%%%%%%%%%%%%%%%%%%%%%%%%%%%%%%%%%%%%%%%%%%%%%%%%%%%%%%%%%%%%%%%%%%%%%%%%%%%%%%%%%%%%%%%%%%%%%%%%%%%%%%%%%%%%%%%%%%%%%%%%%%%%%
\subsection{Predicate reasoning}\label{sec23}
We can now discuss probabilistic reasoning with predicate language. The following property is useful to simply our notation.
\begin{proposition}\label{sec2:prop2}
Let $\alpha\in L$. $p(\alpha=1)=p(\lnot\alpha=0)$.
\end{proposition}
\begin{proof}
$\alpha$ is true in a model iff $\lnot\alpha$ is false in the model. Thus, $\ms{\alpha=1}=\ms{\lnot\alpha=0}$. Using the sum rule, the product rule and Proposition \ref{sec2:prop1}, we have
\begin{align*}
&p(\alpha=1)=\sum_{m_{n}}p(\alpha=1|M=m_{n})p(M=m_{n})=\sum_{m_{n}}\mu^{\ms{\alpha=1}_{m_{n}}}(1-\mu)^{1-\ms{\alpha=1}_{m_{n}}}p(M=m_{n})\\
&=\sum_{m_{n}}\mu^{\ms{\lnot\alpha=0}_{m_{n}}}(1-\mu)^{1-\ms{\lnot\alpha=0}_{m_{n}}}p(M=m_{n})\\
&=\sum_{m_{n}}p(\lnot\alpha=0|M=m_{n})p(M=m_{n})=p(\lnot\alpha=0).
\end{align*}
This holds regardless of the value of $\mu$.
\end{proof}
Therefore, we can write $\lnot\alpha=0$ as $\alpha=1$ and abbreviate $\alpha=1$ as $\alpha$, for all $\alpha\in L$. We also abbreviate $M=m_{n}$ and $D=d_{k}$ as $m_{n}$ and $d_{k}$, respectively.
\par
Now, using the sum rule, the product rule and the conditional independence, i.e., Proposition \ref{sec2:prop1}, the probability of $\alpha,\beta\in L$ can be expressed as follows.
\begin{align*}
p(\alpha,\beta)&=\sum_{k}^{K}\sum_{n}^{N}p(\alpha,\beta,m_{n},d_{k})=\sum_{k}^{K}\sum_{n}^{N}p(\alpha|\beta,m_{n},d_{k})p(\beta|m_{n},d_{k})p(m_{n}|d_{k})p(d_{k})\\
&=\sum_{k}^{K}\sum_{n}^{N}p(\alpha|m_{n})p(\beta|m_{n})p(m_{n}|d_{k})p(d_{k})
\end{align*}
Since $p(d_{k})=1/K$, i.e., Definition \ref{sec3:def01}, and our assumption that each data point supports a single model, we finally have
\begin{eqnarray}
p(\alpha,\beta)=\frac{1}{K}\sum_{k}^{K}\sum_{n}^{N}p(\alpha|m_{n})p(\beta|m_{n})p(m_{n}|d_{k})=\frac{1}{K}\sum_{k}^{K}p(\alpha|m(d_{k}))p(\beta|m(d_{k})).\label{sec33:eq1}
\end{eqnarray}
We here used $\sum_{n}^{N}p(\alpha|m_{n})p(\beta|m_{n})p(m_{n}|d_{k})=p(\alpha|m(d_{k}))p(\beta|m(d_{k}))$. This fact is crucially important in terms of decidability and computational complexity since $N$ can be countably infinite.
\begin{example}\label{sec3:ex6}
Let $\alpha$ be $\forall x(\exists y(Blames(x,y))\to Blames(alice,x))$ shown in Figure \ref{sec3:fig1}. The probability of the formula being true can be evaluated using Equation (\ref{sec33:eq1}).
\begin{eqnarray*}
p(\alpha)&=&\frac{1}{20}\sum_{k=1}^{20}p(\alpha|m(d_{k}))=\frac{1}{20}\sum_{k=1}^{20}\mu^{\ms{\alpha}_{m(d_{k})}}(1-\mu)^{1-\ms{\alpha}_{m(d_{k})}}\\
&=&\frac{1}{20}\big\{\sum_{k\in\{1\text{-}17\}}\mu^{1}(1-\mu)^{0}+\sum_{k\in\{18\text{-}20\}}\mu^{0}(1-\mu)^{1}\big\}\\
&=&\frac{1}{20}\big\{\sum_{k\in\{1\text{-}17\}}\mu+\sum_{k\in\{18\text{-}20\}}(1-\mu)\big\}=\frac{17\mu+3(1-\mu)}{20}
\end{eqnarray*}
Therefore, $p(\alpha)=17/20$ when $\mu=1$. This result is intuitive as it is the number of data supporting models where $\alpha$ is true, out of all the twenty data.
\end{example}
%%%%%%%%%%%%%%%%%%%%%%%%%%%%%%%%%%%%%%%%%%%%%%%%%%%%%%%%%%%%%%%%%%%%%%%%%%%%%%%%%%%%%%%%%%%%%%%%%%%%%%%%%%%%%%%%%%%%%%%%%%%%%%%%%%%%%%%%%%%%%%%%%%%%%%%%%%%%%%%
\section{Evaluations}\label{sec3}
\subsection{Reasoning as learning}\label{sec31}
The common view in statistics is that observed data are generated from unobserved parameters of a probability distribution. Maximum likelihood estimation (MLE) is the most commonly used statistical method to estimate the values of unobserved parameters only from observed data. MLE is defined as $\hat{\Theta}=\argmax_{\Theta}p(d_{1},d_{2},...,d_{K}|\Theta)$, where each $d_{k}$ is an observed data point and $\Theta$ is the set of parameters of a probability distribution.
\begin{proposition}
Let $\{d_{1},d_{2},...,d_{K}\}$ be a multiset of $K$ data and $\Theta$ be the parameters of a categorical distribution. $p(M)=\hat{\Theta}$ if and only if $\hat{\Theta}$ maximises the likelihood of data, i.e., $\hat{\Theta}=\argmax_{\Theta}p(d_{1},d_{2},...,d_{K}|\Theta)$.
\end{proposition}
\begin{proof}
Let $K$ be the total number of data, and $K_{n}$ be the number of data in the $n$th category or model. The maximum likelihood estimate of the parameter $\theta_{n}$ for a categorical distribution is simply known as the relative frequency of data, i.e.,
\begin{eqnarray*}
\theta_{n}=\frac{\text{The number of data in the }n\text{th category}}{\text{The total number of data}}=\frac{K_{n}}{K}.
\end{eqnarray*}
Therefore, the maximum likelihood estimate is given by
\begin{eqnarray*}
\hat{\Theta}=(\frac{K_{1}}{K},\frac{K_{2}}{K},....,\frac{K_{N}}{K}).
\end{eqnarray*}
Let $m_{n}$ be a model of predicate logic. Using the sum and product rules, we have
\begin{align*}
p(m_{n})=\sum_{k}p(m_{n},d_{k})=\sum_{k}p(m_{n}|d_{k})p(d_{k})=\frac{1}{K}\sum_{k}p(m_{n}|d_{k})=\frac{K_{n}}{K}.
\end{align*}
Therefore, we have $p(M)=\hat{\Theta}$.
\end{proof}
\begin{example}\label{sec4:ex1}
Consider the twenty data and thirty two models shown on the top layers of the both hierarchies in Figure \ref{sec3:fig1}. Let $K$ be the total number of data, and $K_{n}$ be the number of data in the $n$th model. We then have $p(m_{1})=\frac{7}{20}$, $p(m_{4})=\frac{4}{20}$, $p(m_{7})=\frac{6}{20}$, $p(m_{11})=\frac{1}{20}$,  $p(m_{13})=\frac{2}{20}$ and $p(m_{n})=0$, for all the remaining models $m_{n}$.
\end{example}
%%%%%%%%%%%%%%%%%%%%%%%%%%%%%%%%%%%%%%%%%%%%%%%%%%%%%%%%%%%%%%%%%%%%%%%%%%%%%%%%%%%%%%%%%%%%%%%%%%%%
\subsection{Reasoning from possible information}\label{sec32}
This section aims to logically characterise the inference of predicative abstraction with $\mu=1$. We focus on the relation between models and formulas by marginalising out data, i.e., $p(\alpha,M)=\sum_{k}p(\alpha,M,D=d_{k})$. As usual, we use the symbol $\ms{\Delta}$ to denote the set of models where all the formulas in $\Delta\subseteq L$ are true, i.e., $\ms{\Delta}=\bigcap_{\alpha\in\Delta}\ms{\alpha}$. A model with a non-zero probability is called possible. We use the symbol $\pms{\Delta}$ to denote the set of possible models where all the formulas in $\Delta\subseteq L$ are true, i.e., $\pms{\Delta}=\{m_{n}\in\ms{\Delta}|p(m_{n})\neq 0\}$. We write $m_{n}\in\ms{\Delta}$ and $m_{n}\in\pms{\Delta}$ as $\ms{\Delta}_{m_{n}}=1$ and $\pms{\Delta}_{m_{n}}=1$, respectively. Note that $\ms{\emptyset}$ and $\pms{\emptyset}$ are the sets of all models and all possible models, respectively. We use the empirical consequence relation originally defined for propositional logic.
\begin{definition}[\cite{kido:24-1}]\label{sec2:def4}
Let $\alpha\in L$ and $\Delta\subseteq L$. $\alpha$ is an empirical consequence of $\Delta$, denoted by $\Delta\ent\alpha$, if $\pms{\Delta}\subseteq\pms{\alpha}$.
\end{definition}
As usual, $\alpha$ is a logical consequence of $\Delta$, denoted by $\Delta\cent\alpha$, if $\ms{\Delta}\subseteq\ms{\alpha}$. The empirical consequence relation $\ent$ is thus a probabilistic generalisation of the logical consequence relation $\cent$.
\par
We write $p_{\mu=1}(\alpha)$ and $p_{\mu\to1}(\alpha)$ when we want to specify the value of $\mu$ used in the evaluation. We often omit it when the value of $\mu$ is obvious from the context. We can now logically characterise the inference of predicative abstraction with $\mu=1$.
\begin{theorem}\label{sec42:the1}
Let $\alpha\in L$ and $\Delta\subseteq L$ such that $\pms{\Delta}\neq\emptyset$.
\begin{eqnarray*}
p_{\mu=1}(\alpha|\Delta)=\frac{\sum_{n}\ms{\alpha}_{m_{n}}\pms{\Delta}_{m_{n}}p(m_{n})}{\sum_{n}\pms{\Delta}_{m_{n}}p(m_{n})}
\end{eqnarray*}
\end{theorem}
\begin{proof}
Using the definition of conditional probability and the conditional independence we showed in the previous section, we have 
\begin{eqnarray*}
&&p(\alpha|\Delta)=\frac{p(\alpha,\Delta)}{p(\Delta)}=\frac{\sum_{n}p(\alpha,\Delta,m_{n})}{\sum_{n}p(\Delta,m_{n})}=\frac{\sum_{n}p(\alpha|m_{n})p(\Delta|m_{n})p(m_{n})}{\sum_{n}p(\Delta|m_{n})p(m_{n})}.
\end{eqnarray*}
Dividing models into possible ones, i.e., $\pms{\Delta}$, and the others, we have
\begin{eqnarray*}
&&p(\alpha|\Delta)=\frac{\textstyle{\sum_{m_{n}\in\pms{\Delta}}p(\alpha|m_{n})p(\Delta|m_{n})p(m_{n})+\sum_{m_{n}\notin\pms{\Delta}}p(\alpha|m_{n})p(\Delta|m_{n})p(m_{n})}}{\textstyle{\sum_{m_{n}\in\pms{\Delta}}p(\Delta|m_{n})p(m_{n})+\sum_{m_{n}\notin\pms{\Delta}}p(\Delta|m_{n})p(m_{n})}}.
\end{eqnarray*}
Since $\mu=1$, $p(\Delta|m_{n})$ can be expanded as follows.
\begin{eqnarray*}
&&p(\Delta|m_{n})=\prod_{\beta\in\Delta}p(\beta|m_{n})=\prod_{\beta\in\Delta}1^{\ms{\beta}_{m_{n}}}0^{1-\ms{\beta}_{m_{n}}}
=
\begin{cases}
1&\text{if } m_{n}\in\ms{\Delta}\\
0&\text{otherwise}
\end{cases}
\end{eqnarray*}
Thus, $p(\Delta|m_{n})=0$, for all $m_{n}\notin\ms{\Delta}$. Moreover, $p(m_{n})=0$, for all $m_{n}\in\ms{\Delta}\setminus\pms{\Delta}$. These two facts imply that $\sum_{m_{n}\notin\pms{\Delta}}p(\Delta|m_{n})p(m_{n})=0$. Since $p(\Delta|m_{n})=1$, for all $m_{n}\in\pms{\Delta}$, we have
\begin{eqnarray*}
&&p(\alpha|\Delta)=\frac{\sum_{m_{n}\in\pms{\Delta}}p(\alpha|m_{n})p(m_{n})}{\sum_{m_{n}\in\pms{\Delta}}p(m_{n})}=\frac{\sum_{n}\pms{\Delta}_{m_{n}}p(\alpha|m_{n})p(m_{n})}{\sum_{n}\pms{\Delta}_{m_{n}}p(m_{n})}.
\end{eqnarray*}
Since $\mu=1$, $p(\alpha|m_{n})$ can be developed as follows.
\begin{eqnarray*}
&&p(\alpha|m_{n})=1^{\ms{\alpha}_{m_{n}}}0^{1-\ms{\alpha}_{m_{n}}}
=
\begin{cases}
1&\text{if } m_{n}\in\ms{\alpha}\\
0&\text{otherwise}
\end{cases}
\end{eqnarray*}
Therefore, 
\begin{eqnarray*}
&&p(\alpha|\Delta)=\frac{\sum_{n}\ms{\alpha}_{m_{n}}\pms{\Delta}_{m_{n}}p(m_{n})}{\sum_{n}\pms{\Delta}_{m_{n}}p(m_{n})}.
\end{eqnarray*}
\end{proof}
The assumption of $\pms{\Delta}\neq\emptyset$ guarantees that $p_{\mu=1}(\alpha|\Delta)$ involves no division by zero, which causes an undefined value.
\begin{example}[Continued.]\label{sec4:ex2}
Given that everyone blames someone and that everyone is blamed by someone, what is the probability that Alice blames Bob? Each model determines the truth value of each atomic formula as follows (see Figure \ref{sec3:fig1}). 
\begin{eqnarray*}
&&\pms{Blames(alice,alice)}=\{m_{7},m_{13}\}\\
&&\pms{Blames(alice,bob)}=\{m_{4},m_{7},m_{11}\}\\
&&\pms{Blames(bob,alice)}=\{m_{4}\}\\
&&\pms{Blames(bob,bob)}=\{m_{11},m_{13}\}
\end{eqnarray*}
Thus, each compound formula has the following truth value.
\begin{eqnarray*}
&&\pms{\forall x\exists y~Blames(x,y)}_{m_{n}}=\min_{c_{1}\in{\cal C}}\big\{\max_{c_{2}\in{\cal C}}\big\{\pms{Blames(c_{1},c_{2})}_{m_{n}}\big\}\big\}\\
&&~~~~~=\min\left\{\begin{array}{l}
\max\big\{\pms{Blames(alice,alice)}_{m_{n}}, \pms{Blames(alice,bob)}_{m_{n}}\big\},\\
\max\big\{\pms{Blames(bob,alice)}_{m_{n}}, \pms{Blames(bob,bob)}_{m_{n}}\big\}
\end{array}\right\}\\
&&~~~~~=
\begin{cases}
1&\text{if } n\in\{4,11,13\}\\
0&\text{otherwise }
\end{cases}
\\
&&\pms{\forall y\exists x~Blames(x,y)}_{m_{n}}=\min_{c_{1}\in{\cal C}}\big\{\max_{c_{2}\in{\cal C}}\big\{\pms{Blames(c_{2},c_{1})}_{m_{n}}\big\}\big\}\\
&&~~~~~=\min\left\{\begin{array}{l}
\max\big\{\pms{Blames(alice,alice)}_{m_{n}}, \pms{Blames(bob,alice)}_{m_{n}}\big\},\\
\max\big\{\pms{Blames(alice,bob)}_{m_{n}}, \pms{Blames(bob,bob)}_{m_{n}}\big\}
\end{array}\right\}\\
&&~~~~~=
\begin{cases}
1&\text{if } n\in\{4,7,13\}\\
0&\text{otherwise }
\end{cases}
\end{eqnarray*}
Now, we can apply Theorem \ref{sec42:the1} to find the probability.
\begin{eqnarray*}
&&p(Blames(alice,bob)|\forall x\exists y~Blames(x,y), \forall y\exists x~Blames(x,y))\\
&&=\frac{\sum_{n}\ms{Blames(alice,bob)}_{m_{n}}\pms{\forall x\exists y~Blames(x,y), \forall y\exists x~Blames(x,y)}_{m_{n}}p(m_{n})}{\sum_{n}\pms{\forall x\exists y~Blames(x,y), \forall y\exists x~Blames(x,y)}_{m_{n}}p(m_{n})}\\
&&=\frac{\sum_{n\in\{4\}}p(m_{n})}{\sum_{n\in\{4,13\}}p(m_{n})}=\frac{\sum_{k}p(m_{4}|d_{k})}{\sum_{k}\{p(m_{4}|d_{k})+p(m_{13}|d_{k})\}}=\frac{4}{4+2}=\frac{2}{3}
\end{eqnarray*}
We added the marginalised data in the last line. $p(d_{k})$ is canceled out due to its uniformity.
\end{example}
The following fact shows that Theorem \ref{sec42:the1} is a generalisation of the empirical consequence relation. The proof can be found in Appendix \ref{sec:appendix}.
\begin{corollary}\label{sec42:cor1}
Let $\alpha\in L$ and $\Delta\subseteq L$ such that $\pms{\Delta}\neq\emptyset$. $p_{\mu=1}(\alpha|\Delta)=1$ if and only if $\Delta\ent\alpha$.
\end{corollary}
%%%%%%%%%%%%%%%%%%%%%%%%%%%%%%%%%%%%%%%%%%%%%%%%%%%%%%%%%%%%%%%%%%%%%%%%%%%%%%%%%%%%%%%%%%%%%%%%%%%%%%%%%%%%%%%%%%%%%%%%%%%%%%%%%%%%%%%%%%%%%%%%%%%%%%%%%%%%%%%%%%%%%%%%%%%%%%%%
\subsection{Reasoning from impossible information}\label{sec33}
This section aims to logically characterise the inference of predicative abstraction with $\mu\to 1$. We use the maximal possible set originally defined for propositional logic.
\begin{definition}[\cite{kido:24-1}]\label{sec2:def5}
Let $\Delta\subseteq L$. $S\subseteq\Delta$ is a maximal possible subset of $\Delta$ if $\pms{S}\neq\emptyset$ and $\pms{S\cup\{\alpha\}}=\emptyset$, for all $\alpha\in\Delta\setminus S$.
\end{definition}
As usual, $S\subseteq\Delta$ is a maximal consistent subset of $\Delta$ if $\ms{S}\neq\emptyset$ and $\ms{S\cup\{\alpha\}}=\emptyset$, for all $\alpha\in\Delta\setminus S$. A maximal possible set is thus a probabilistic generalisation of a maximal consistent set. We use the symbol $MPS(\Delta)$ and $MCS(\Delta)$ to denote the set of the cardinality-maximal possible subsets of $\Delta$ and the set of cardinality-maximal consistent subsets of $\Delta$, respectively.
\begin{example}\label{sec2:ex4}
Let us discuss examples of maximal consistent sets and maximal possible sets in propositional logic. Consider the hierarchy shown on the left in Figure \ref{sec3:fig1}. Let $\Delta=\{rain, sprinkler, sprinkler\to wet, hot, wet, \lnot wet\}$. The set of the maximal consistent subsets of $\Delta$ is $\{S_{1}, S_{2}, S_{3}\}$ given as follows.
\begin{itemize}
\item $S_{1}=\{rain, sprinkler, sprinkler\to wet, hot, wet\}$ where $\ms{S_{1}}=\{m_{6},m_{8}\}$
\item $S_{2}=\{rain, sprinkler, hot, \lnot wet\}$ where $\ms{S_{2}}=\{m_{14},m_{16}\}$
\item $S_{3}=\{rain, sprinkler\to wet, hot, \lnot wet\}$ where $\ms{S_{3}}=\{m_{10},m_{12}\}$
\end{itemize}
Therefore, the set of the cardinality-maximal consistent subsets of $\Delta$ is $\{S_{1}\}$, i.e., $MCS(\Delta)=\{S_{1}\}$. Meanwhile, the set of the maximal possible subsets of $\Delta$ is $\{S_{4}, S_{5}, S_{6}\}$ given as follows.
\begin{itemize}
\item $S_{4}=\{rain, sprinkler\to wet, hot, wet\}$ where $\pms{S_{4}}=\{m_{4}\}$
\item $S_{5}=\{sprinkler, sprinkler\to wet, hot, wet\}$ where $\pms{S_{5}}=\{m_{24}\}$
\item $S_{6}=\{sprinkler\to wet, hot, \lnot wet\}$ where $\pms{S_{6}}=\{m_{26},m_{28}\}$
\end{itemize}
Therefore, the set of the cardinality-maximal possible subsets of $\Delta$ is $\{S_{4},S_{5}\}$, i.e., $MPS(\Delta)=\{S_{4},S_{5}\}$.
\end{example}
We often use the symbol $\pams{\Delta}$ to denote the set of possible models where all the formulas in a cardinality-maximal possible subset of $\Delta$ are true, i.e., $\pams{\Delta}=\bigcup_{S\in MPS(\Delta)}\pms{S}$. We also use the symbol $\ams{\Delta}$ to denote the set of models where all the formulas in a cardinality-maximal consistent subset of $\Delta$ are true, i.e., $\ams{\Delta}=\bigcup_{S\in MCS(\Delta)}\ms{S}$. 
\begin{theorem}\label{sec4:the2} 
Let $\alpha\in L$ and $\Delta\subseteq L$.
\begin{eqnarray*}
p_{\mu\to1}(\alpha|\Delta)=\frac{\sum_{n}\ms{\alpha}_{m_{n}}(\bigcup_{S\in MPS(\Delta)}\pms{S})_{m_{n}}p(m_{n})}{\sum_{n}(\bigcup_{S\in MPS(\Delta)}\pms{S})_{m_{n}}p(m_{n})}
\end{eqnarray*}
\end{theorem}
\begin{proof}
Let the symbols $|\Delta|$ and $|\Delta|_{m_{n}}$ denote the number of formulas in $\Delta$ and the number of formulas in $\Delta$ that are true in the model $m_{n}$, i.e., $|\Delta|_{m_{n}}=\sum_{\beta\in\Delta}\m{\beta}_{m_{n}}$, respectively. Using the definition of conditional probability and the conditional independence we showed in the previous section, we have
\begin{eqnarray*}
&&p(\alpha|\Delta)=\frac{p(\alpha,\Delta)}{p(\Delta)}=\frac{\sum_{n}p(\alpha,\Delta,m_{n})}{\sum_{n}p(\Delta,m_{n})}=\frac{\sum_{n}p(\alpha|m_{n})p(\Delta|m_{n})p(m_{n})}{\sum_{n}p(\Delta|m_{n})p(m_{n})}.
\end{eqnarray*}
Dividing models into ones in $\pams{\Delta}$ and the others, we have
\begin{eqnarray*}
&&p(\alpha|\Delta)=\frac{\textstyle{\sum_{\hat{m}_{n}\in\pams{\Delta}}p(\alpha|\hat{m}_{n})p(\Delta|\hat{m}_{n})p(\hat{m}_{n})+\sum_{m_{n}\notin\pams{\Delta}}p(\alpha|m_{n})p(\Delta|m_{n})p(m_{n})}}{\textstyle{\sum_{\hat{m}_{n}\in\pams{\Delta}}p(\Delta|\hat{m}_{n})p(\hat{m}_{n})+\sum_{m_{n}\notin\pams{\Delta}}p(\Delta|m_{n})p(m_{n})}}.
\end{eqnarray*}
$p(\Delta|m_{n})$ can be developed as follows, for all models $m_{n}$.
\begin{align*}
p(\Delta|m_{n})&=\prod_{\beta\in\Delta}p(\beta|m_{n})=\prod_{\beta\in\Delta}\mu^{\ms{\beta}_{m_{n}}}(1-\mu)^{1-\ms{\beta}_{m_{n}}}\\
&=\mu^{\sum_{\beta\in\Delta}\ms{\beta}_{m_{n}}}(1-\mu)^{\sum_{\beta\in\Delta}(1-\ms{\beta}_{m_{n}})}=\mu^{|\Delta|_{m_{n}}}(1-\mu)^{|\Delta|-|\Delta|_{m_{n}}}
\end{align*}
Therefore, $p(\alpha|\Delta)=\lim_{\mu\rightarrow 1}\frac{W+X}{Y+Z}$ where
\begin{align*}
W&=\textstyle{\sum_{\hat{m}_{n}\in\pams{\Delta}}p(\alpha|\hat{m}_{n})\mu^{|\Delta|_{\hat{m}_{n}}}(1-\mu)^{|\Delta|-|\Delta|_{\hat{m}_{n}}}p(\hat{m}_{n})}\\
X&=\textstyle{\sum_{m_{n}\notin\pams{\Delta}}p(\alpha|m_{n})\mu^{|\Delta|_{m_{n}}}(1-\mu)^{|\Delta|-|\Delta|_{m_{n}}}p(m_{n})}\\
Y&=\textstyle{\sum_{\hat{m}_{n}\in\pams{\Delta}}\mu^{|\Delta|_{\hat{m}_{n}}}(1-\mu)^{|\Delta|-|\Delta|_{\hat{m}_{n}}}p(\hat{m}_{n})}\\
Z&=\textstyle{\sum_{m\notin\pams{\Delta}}\mu^{|\Delta|_{m_{n}}}(1-\mu)^{|\Delta|-|\Delta|_{m_{n}}}p(m_{n})}.
\end{align*}
If $m_{n}\notin\pams{\Delta}$ then $m_{n}$ is impossible or $m_{n}$ is a possible model of a subset of $\Delta$ that is not a cardinality-maximal possible subset of $\Delta$. Therefore, $p(m_{n})=0$ or there is $\hat{m}_{n}\in\pams{\Delta}$ such that $|\Delta|_{m_{n}}<|\Delta|_{\hat{m}_{n}}$. $|\Delta|_{\hat{m}_{1}}=|\Delta|_{\hat{m}_{2}}$ by definition, for all $\hat{m}_{1},\hat{m}_{2}\in\pams{\Delta}$. The fraction thus can be simplified by dividing the denominator and numerator by $(1-\mu)^{|\Delta|-|\Delta|_{\hat{m}_{n}}}$. We thus have $p(\alpha|\Delta)=\lim_{\mu\rightarrow 1}\frac{W'+X'}{Y'+Z'}$ where
\begin{align*}
W'&=\textstyle{\sum_{\hat{m}_{n}\in\pams{\Delta}}p(\alpha|\hat{m}_{n})\mu^{|\Delta|_{\hat{m}_{n}}}p(\hat{m}_{n})}\\
X'&=\textstyle{\sum_{m_{n}\notin\pams{\Delta}}p(\alpha|m_{n})\mu^{|\Delta|_{m_{n}}}(1-\mu)^{|\Delta|_{\hat{m}_{n}}-|\Delta|_{m_{n}}}p(m_{n})}\\
Y'&=\textstyle{\sum_{\hat{m}_{n}\in\pams{\Delta}}\mu^{|\Delta|_{\hat{m}_{n}}}p(\hat{m}_{n})}\\
Z'&=\textstyle{\sum_{m_{n}\notin\pams{\Delta}}\mu^{|\Delta|_{m_{n}}}(1-\mu)^{|\Delta|_{\hat{m}_{n}}-|\Delta|_{m_{n}}}p(m_{n})}.
\end{align*}
Applying the limit, we can cancel out $X'$ and $Z'$.
\begin{align*}
p(\alpha|\Delta)=\lim_{\mu\to1}\frac{\textstyle{\sum_{\hat{m}_{n}\in\pams{\Delta}}p(\alpha|\hat{m}_{n})p(\hat{m}_{n})}}{\textstyle{\sum_{\hat{m}_{n}\in\pams{\Delta}}p(\hat{m}_{n})}}=\frac{\textstyle{\sum_{\hat{m}_{n}\in\pams{\Delta}}1^{\ms{\alpha}_{\hat{m}}}0^{1-\ms{\alpha}_{\hat{m}_{n}}}p(\hat{m}_{n})}}{\textstyle{\sum_{\hat{m}_{n}\in\pams{\Delta}}p(\hat{m}_{n})}}.
\end{align*}
By convention, $1^{\ms{\alpha}_{\hat{m}_{n}}}0^{1-\ms{\alpha}_{\hat{m}_{n}}}=1^{1}0^{0}=1$ if $\hat{m}_{n}\in\ms{\alpha}$ and $1^{\ms{\alpha}_{\hat{m}_{n}}}0^{1-\ms{\alpha}_{\hat{m}_{n}}}=1^{0}0^{1}=0$ otherwise. Therefore,
\begin{align}\label{sec4:eq1}
p(\alpha|\Delta)=\frac{\textstyle{\sum_{\hat{m}_{n}\in\pams{\Delta}}\ms{\alpha}_{\hat{m}_{n}}p(\hat{m}_{n})}}{\textstyle{\sum_{\hat{m}_{n}\in\pams{\Delta}}p(\hat{m}_{n})}}=\frac{\textstyle{\sum_{n}\ms{\alpha}_{m_{n}}\pams{\Delta}_{m_{n}}p(m_{n})}}{\textstyle{\sum_{n}\pams{\Delta}_{m_{n}}p(m_{n})}}.
\end{align}
\end{proof}
The denominator of Theorem \ref{sec4:the2} cannot be zero. For example, $\pams{\emptyset}=\bigcup_{S\in MPS(\emptyset)}\pms{S}=\pms{\emptyset}$, where $\pms{\emptyset}$ is the set of all possible models. All models cannot be impossible by definition.
\begin{example}[Continued.]\label{sec4:ex3}
Given that everyone blames someone, everyone is blamed by someone, and someone blames everyone, what is the probability that Alice blames Bob? Let us use $\alpha$ to denote $Blames(x,y)$. Each compound formula has the following truth value.
\begin{align*}
&\pms{\forall x\exists y~\alpha}=\{m_{4},m_{11},m_{13}\}\text{ (See Example \ref{sec4:ex2}.)}\\
&\pms{\forall y\exists x~\alpha}=\{m_{4},m_{7},m_{13}\}\text{ (See Example \ref{sec4:ex2}.)}\\
&\pms{\exists x\forall y~\alpha}_{m_{n}}=\max_{c_{1}\in{\cal C}}\big\{\pms{\forall y~Blames(c_{1},y)}_{m_{n}}\big\}=\max_{c_{1}\in{\cal C}}\big\{\min_{c_{2}\in{\cal C}}\big\{\pms{Blames(c_{1},c_{2})}_{m_{n}}\big\}\big\}\\
&=\max\left\{\begin{array}{l}
\min\big\{\pms{Blames(alice,alice)}_{m_{n}}, \pms{Blames(alice,bob)}_{m_{n}}\big\},\\
\min\big\{\pms{Blames(bob,alice)}_{m_{n}}, \pms{Blames(bob,bob)}_{m_{n}}\big\}
\end{array}\right\}=
\begin{cases}
1&\text{if } n\in\{7\}\\
0&\text{otherwise }
\end{cases}
\end{align*}
Now, we can apply Theorem \ref{sec4:the2} to find the probability.
\begin{eqnarray*}
&&p(Blames(alice,bob)|\forall x\exists y~\alpha, \forall y\exists x~\alpha, \exists x\forall y~\alpha)\\
&=&\frac{\textstyle{\sum_{n}\ms{Blames(alice,bob)}_{m_{n}}\pams{\forall x\exists y~\alpha, \forall y\exists x~\alpha, \exists x\forall y~\alpha}_{m_{n}}p(m_{n})}}{\textstyle{\sum_{n}\pams{\forall x\exists y~\alpha, \forall y\exists x~\alpha, \exists x\forall y~\alpha}_{m_{n}}p(m_{n})}}\\
&=&\frac{\textstyle{\sum_{n}\ms{Blames(alice,bob)}_{m_{n}}(\pms{\forall x\exists y~\alpha, \forall y\exists x~\alpha}\cup\pms{\forall y\exists x~\alpha, \exists x\forall y~\alpha})_{m_{n}}p(m_{n})}}{\textstyle{\sum_{n}(\pms{\forall x\exists y~\alpha, \forall y\exists x~\alpha}\cup\pms{\forall y\exists x~\alpha, \exists x\forall y~\alpha})_{m_{n}}p(m_{n})}}\\
&=&\frac{\textstyle{\sum_{n\in\{4,7\}}p(m_{n})}}{\textstyle{\sum_{n\in\{4,7,13\}}p(m_{n})}}=\frac{\sum_{k}\{p(m_{4}|d_{k})+p(m_{7}|d_{k})\}}{\sum_{k}\{p(m_{4}|d_{k})+p(m_{7}|d_{k})+p(m_{13}|d_{k})\}}=\frac{4+6}{4+6+2}=\frac{5}{6}
\end{eqnarray*}
We added the marginalised data in the last line. $p(d_{k})$ is canceled out due to its uniformity.
\end{example}
\par
Theorem \ref{sec4:the2} implies the following properties, and their proofs can be found in Appendix \ref{sec:appendix}. The certain reasoning, i.e., reasoning with a probability of one, can be characterised in terms of the empirical consequence relation and its application to cardinality-maximal possible sets.
\begin{corollary}\label{sec4:cor1} 
Let $\alpha\in L$ and $\Delta\subseteq L$. $p_{\mu\to1}(\alpha|\Delta)=1$ if and only if $S\ent\alpha$, for all cardinality-maximal possible subsets $S$ of $\Delta$.
\end{corollary}
Let us assume $\pams{\Delta}=\ams{\Delta}$. This is the case when all the models satisfying $S$ are possible, for all cardinality-maximal consistent subsets $S$ of $\Delta$, i.e., $\pms{S}=\ms{S}$, for all $S\in MCS(\Delta)$. Then, the certain reasoning can be characterised in terms of the logical consequence relation and its application to cardinality-maximal consistent sets.
\begin{corollary}\label{sec4:cor2} 
Let $\alpha\in L$ and $\Delta\subseteq L$ such that $\pams{\Delta}=\ams{\Delta}$. $p_{\mu\to1}(\alpha|\Delta)=1$ if and only if $S\cent\alpha$, for all cardinality-maximal consistent subsets $S$ of $\Delta$.
\end{corollary}
Let us assume that there is a possible model satisfying $\Delta$, i.e., $\pms{\Delta}\neq\emptyset$. Then, the certain reasoning can be characterised in terms of the empirical consequence relation.
\begin{corollary}\label{sec4:cor3} 
Let $\alpha\in L$ and $\Delta\subseteq L$ such that $\pms{\Delta}\neq\emptyset$. $p_{\mu\to1}(\alpha|\Delta)=1$ if and only if $\Delta\ent\alpha$.
\end{corollary}
Let us assume $\pams{\Delta}=\ams{\Delta}$ and $\pms{\Delta}\neq\emptyset$. The certain reasoning can then be characterised in terms of the logical consequence relation.
\begin{corollary}\label{sec4:cor4} 
Let $\alpha\in L$ and $\Delta\subseteq L$ such that $\pams{\Delta}=\ams{\Delta}$ and $\pms{\Delta}\neq\emptyset$. $p_{\mu\to1}(\alpha|\Delta)=1$ if and only if $\Delta\cent\alpha$.
\end{corollary}
The following fact shows that the assumption of $\pms{\Delta}\neq\emptyset$ makes $\mu=1$ and $\mu\to 1$ equivalent.
\begin{corollary}\label{sec4:cor5} 
Let $\alpha\in L$ and $\Delta\subseteq L$ such that $\pms{\Delta}\neq\emptyset$. $p_{\mu\to1}(\alpha|\Delta)=p_{\mu=1}(\alpha|\Delta)$.
\end{corollary}
The following is the summary of Sections \ref{sec32} and \ref{sec33}. There is an arrow from property $x$ to property $y$ when $x$ is more general than $y$. Each arrow specifies the assumption of $y$ that makes it less general than $x$.
\begin{eqnarray*}
&\text{Theorem \ref{sec42:the1}} \xrightarrow[]{p(\alpha|\Delta)=1}& \text{Corollary \ref{sec42:cor1}}\\
&\text{Theorem \ref{sec4:the2}} \xrightarrow[]{p(\alpha|\Delta)=1}& \text{Corollary \ref{sec4:cor1}} \xrightarrow[]{\pams{\Delta}=\ams{\Delta}} \text{Corollary \ref{sec4:cor2}} \xrightarrow[]{\pms{\Delta}\neq\emptyset} \text{Corollary \ref{sec4:cor4}}\\
&&\text{Corollary \ref{sec4:cor1}} \xrightarrow[]{\pms{\Delta}\neq\emptyset} \text{Corollary \ref{sec4:cor3}} \xrightarrow[]{\pams{\Delta}=\ams{\Delta}} \text{Corollary \ref{sec4:cor4}}
\end{eqnarray*}
%%%%%%%%%%%%%%%%%%%%%%%%%%%%%%%%%%%%%%%%%%%%%%%%%%%%%%%%%%%%%%%%%%%%%%%%%%%%%%%%%%%%%%%%%%%%%%%%%%%%%%%%%%%%%%%%%%%%%%%%%%%%%%%%%%%%%%%%%%%%%%%%%%%%%%%%%%
\subsection{Applicability}\label{sec34}
This section aims to show the applicability of the inference of predicative abstraction. We discuss a solution to a simple example, similar to Examples  \ref{sec1:ex1} and \ref{sec1:ex2}, that is often difficult to solve by existing AI approaches. Consider the following simple arithmetic quiz for testing one's abstract and logical thinking skills.
\begin{example}\label{sec4.4:ex1}
Which one of the following data $d_{3}$, $d_{4}$ and $d_{5}$ is a companion of $d_{1}$ and $d_{2}$?
\begin{center}
\includegraphics[scale=0.3]{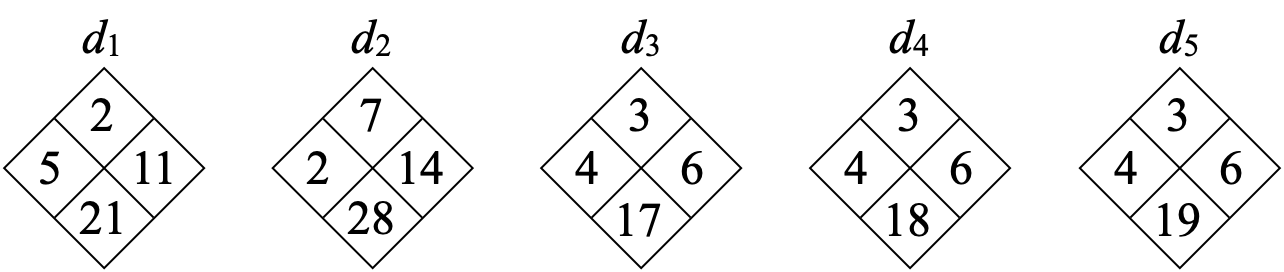}
\end{center}
\end{example}
The correct answer could be $d_{4}$. Only $d_{1}$, $d_{2}$ and $d_{4}$ satisfy the following arithmetic rule.
\begin{align*}
\textit{top}\times \textit{left} + \textit{right} = \textit{bottom}
\end{align*}
We assume the following vocabularies for the usual arithmetic operations.
\begin{eqnarray*}
&&{\cal C}=\{\text{`}i\text{'}, \textit{left, right, top, bottom} | i\in\mathbb{R}\}\\
&&{\cal F}=\{\text{`}+\text{'}, \text{`}-\text{'}, \text{`}\times\text{'}, \text{`}\div\text{'}\}\text{, where }ar(\cdot)=2\text{, for all `}\cdot\text{'}\in{\cal F}\\
&&{\cal P}=\{\text{`}=\text{'}\}\text{, where }ar(\text{`}=\text{'})=2
\end{eqnarray*}
${\cal C}$ comprises the constants for the real numbers $i\in \mathbb{R}$ and for the four numbers located to the left, right, top and bottom. We assume no variables, i.e., ${\cal V}=\emptyset$. Here, $\textit{left, right, top, bottom}$ are not variables since each of them has a single value for each data point. We omit the symbol `' when it is obvious that its inside is a vocabulary of a predicate language. We use the infix notation for readability. Let $w$, $x$, $y$ and $z$ be the numbers on the left, right, top and bottom of a data point, respectively. Since we are interested in the usual arithmetic operations, we assume the function $m$ that maps each $d_{k}$ to the model $m(d_{k})=\langle u(m(d_{k})), v(m(d_{k}))\rangle$ given as follows.
\begin{align*}
&u(m(d_{k}))=\mathbb{R}\\
&v(m(d_{k}))(\text{`}i\text{'})=i\text{, for all }i\in u(m(d_{k}))\\
&v(m(d_{k}))(\textit{left})=w\\
&v(m(d_{k}))(\textit{right})=x\\
&v(m(d_{k}))(\textit{top})=y\\
&v(m(d_{k}))(\textit{bottom})=z\\
&v(m(d_{k}))(\text{`}\cdot\text{'})(i,j)=i\cdot j\text{, for all }\text{`}\cdot\text{'}\in{\cal F}\\
&v(m(d_{k}))(\text{`='})=\{(i,i)|i\in u(m(d_{k}))\}
\end{align*}
Note that only the assignments to \textit{left}, \textit{right}, \textit{top}, \textit{bottom} depend on individual data, and the other assignments depend only on the arithmetic rule.
\par
Our solution is based on two modes of the inference of abstraction. The first mode is the inference over $d_{1}$ and $d_{2}$, which aims to extract an abstract rule from the concrete data using the following equation.
\begin{eqnarray}\label{sec44:sol1}
\hat{\alpha}&=&\argmax_{\alpha\in L}p(\alpha)
\end{eqnarray}
The second model is the inference over $d_{3}$, $d_{4}$ and $d_{5}$, which aims to apply the extracted abstract rule to the concrete data using the following equation.
\begin{eqnarray}\label{sec44:sol2}
\hat{\beta}&=&\argmax_{\beta\in L}p(\beta|\hat{\alpha})
\end{eqnarray}
Equations (\ref{sec44:sol1}) and (\ref{sec44:sol2}) are computationally intractable, since the predicate language generally has the infinite number of formulas. It is beyond the scope of this paper to fully discuss how to explore the infinite language space. We here simply search for a formula with a small number of logical connectives for a simple explanation, which conforms to Occam's razor. Now, let $\mu=1$. $\hat{\alpha}=(\textit{top}\times \textit{left} + \textit{right} = \textit{bottom})$ satisfies Equation (\ref{sec44:sol1}).
\begin{eqnarray*}
p(\hat{\alpha})&=&\sum_{n=1}^{\infty}\sum_{k=1}^{2}p(\hat{\alpha},m_{n},d_{k})=\sum_{n=1}^{\infty}\sum_{k=1}^{2}p(\hat{\alpha}|m_{n})p(m_{n}|d_{k})p(d_{k})\\
&=&\sum_{k=1}^{2}p(\hat{\alpha}|m(d_{k}))p(d_{k})=\frac{1}{2}\sum_{k=1}^{2}\ms{\hat{\alpha}}_{m(d_{k})}=1
\end{eqnarray*}
Now, $\hat{\beta}=(\textit{bottom}=18)$ satisfies Equation (\ref{sec44:sol2}).
\begin{eqnarray*}
p(\hat{\beta}|\hat{\alpha})&=&\frac{\sum_{n=1}^{\infty}\sum_{k=3}^{5}p(\hat{\beta},\hat{\alpha},m_{n},d_{k})}{\sum_{n=1}^{\infty}\sum_{k=3}^{5}p(\hat{\alpha},m_{n},d_{k})}=\frac{\sum_{n=1}^{\infty}\sum_{k=3}^{5}p(\hat{\beta}|m_{n})p(\hat{\alpha}|m_{n})p(m_{n}|d_{k})p(d_{k})}{\sum_{n=1}^{\infty}\sum_{k=3}^{5}p(\hat{\alpha}|m_{n})p(m_{n}|d_{k})p(d_{k})}\\
&=&\frac{\sum_{k=3}^{5}p(\hat{\beta}|m(d_{k}))p(\hat{\alpha}|m(d_{k}))}{\sum_{k=3}^{5}p(\hat{\alpha}|m(d_{k}))}=\frac{\sum_{k=3}^{5}\ms{\hat{\beta}}_{m(d_{k})}\ms{\hat{\alpha}}_{m(d_{k})}}{\sum_{k=3}^{5}\ms{\hat{\alpha}}_{m(d_{k})}}=\frac{1}{1}
\end{eqnarray*}
%%%%%%%%%%%%%%%%%%%%%%%%%%%%%%%%%%%%%%%%%%%%%%%%%%%%%%%%%%%%%%%%%%%%%%%%%%%%%%%%%%%%%%%%%%%%%%%%%%%%%%%%%%%%%%%%%%%%%%%%%%%%%%%%%%%%%%%%%%%%%%%%%%%%%%%%%%%%%%%%%%%%%%%%%%%%%%%%%%%%%%%%%%%%%%%%%%%%%%%%%%%%%%%%%%%%%%%%%%%%%%%%%%%%%%%%%%%%%%%%%%%%%%%%%%%%%%%%%%%%%%%%%%%%%%%%%%%%%%%%%%%%%%%%%%%%%%%%%%%%%%%%
\section{Conclusions and Future Work}\label{sec4}
In this paper, we asked how predicate reasoning should be grounded on data for meaningful abstraction. We proposed the inference of abstraction by simply modelling the idea that an intrinsically abstract predicate formula is a selective ignorance of the models of a predicate language and each of the models is a selective ignorance of intrinsically concrete observed data. We showed that the idea is not only enough to characterise the logical consequence relation of predicate logic but also to generalise it for the empirical consequence relation and its application to cardinality-maximal possible sets. The simple yet unconventional idea suggests a fresh perspective to rethink various important issues such as symbol grounding, computationally tractable predicate reasoning, commonsense and reasoning from an inconsistent and impossible source of information.
\par
An important challenge is to integrate connectionism to learn the function $m$ so that it accurately maps data to models (recall the top and middle layers in Figure \ref{sec3:fig1}). A more fundamental question is how an agent should decide which vocabularies to use in its predicate language for creative abstractions of data.
%%%%%%%%%%%%%%%%%%%%%%%%%%%%%%%%%%%%%%%%%%%%%%%%%%%%%%%%%%%%%%%%%%%%%%%%%%%%%%%%%%%%%%%%%%%%%%%%%%%%%%%%%%%%%%%%%%%%%%%%%%%%%%%%%%%%%%%%%%%%%%%%%%%%%%%%%%%%%%%%%%%%%%%%%%%%%%%%%%%%%%%%%%%%%%%%%%%%%%%%%%%%%%%%%%%%%%%%%%%%%%%%%%%%%%%%%%%%%%%%%%%%%%%%%%%%%%%%%%%%%%%%%%%%%%%%%%%%%%%%%%%%%%%%%%%%%%%%%%%%%%%%
\bibliography{btx_kido}

\begin{thebibliography}{10}

\bibitem{ChatGPT}
Open AI.
\newblock Chatgpt.
\newblock https://chat.openai.com/chat, Retrieved Feb 2025., 2025.

\bibitem{Bishop:06}
Christopher~M. Bishop.
\newblock {\em Pattern Recognition and Machine Learning}.
\newblock Springer New York, NY, 1 New York Plaza, Suite 4600, New York, NY
  10004-1562, 2006.

\bibitem{Brewka:91}
Gerhard Brewka.
\newblock {\em Nonmonotonic Reasoning: Logical Foundations of Commonsense}.
\newblock Cambridge University Press, Cambridge, England, 1991.

\bibitem{Dasgupta:20}
Ishita Dasgupta, Eric Schulz, Joshua~B Tenenbaum, and Samuel~J Gershman.
\newblock A theory of learning to infer.
\newblock {\em Psychol Rev.}, 127(3):412--441, 2020.

\bibitem{Davis:15}
Ernest Davis and Gary Marcus.
\newblock Commonsense reasoning and commonsense knowledge in artificial
  intelligence.
\newblock {\em Communications of the ACM}, 58(9):92--103, 2015.

\bibitem{Chollet:19}
DeepSeek-AI.
\newblock On the measure of intelligence.
\newblock https://arxiv.org/abs/1911.01547v2, Retrieved Jan 2025., 2019.

\bibitem{DeepSeek}
DeepSeek-AI.
\newblock Deepseek-r1: Incentivizing reasoning capability in llms via
  reinforcement learning.
\newblock https://arxiv.org/pdf/2501.12948, Retrieved Feb 2025., 2025.

\bibitem{friston:10}
Karl Friston.
\newblock The free-energy principle: a unified brain theory?
\newblock {\em Nature Reviews Neuroscience}, 11:127--138, 2010.

\bibitem{Gregory:97}
Richard~L Gregory.
\newblock Knowledge in perception and illusion.
\newblock {\em Philos Trans R Soc Lond B Biol Sci}, 352(1358):1121--1127, 1997.

\bibitem{Harnad:90}
Stevan Harnad.
\newblock The symbol grounding problem.
\newblock {\em Physica D: Nonlinear Phenomena}, 42(1-3):335--346, 1990.

\bibitem{Hawkins:21}
Jeff Hawkins.
\newblock {\em A Thousand Brains: A New Theory of Intelligence}.
\newblock Basic Books, New York City, 2021.

\bibitem{kido:24-1}
Hiroyuki Kido.
\newblock Inference of abstraction for human-like logical reasoning (in press).
\newblock In {\em The 10th Int Conf on machine Learning, Optimization \& Data
  science - LOD and Symposium on Artificial Intelligence \& Neuroscience (ACAIN
  2024)}, 2024.

\bibitem{kido:24-2}
Hiroyuki Kido.
\newblock Inference of abstraction for human-like probabilistic reasoning (in
  press).
\newblock In {\em The 10th Int Conf on machine Learning, Optimization \& Data
  science - LOD and Symposium on Artificial Intelligence \& Neuroscience (ACAIN
  2024)}, 2024.

\bibitem{Lake:15}
Brenden~M. Lake, Ruslan Salakhutdinov, and Joshua~B. Tenenbaum.
\newblock Human-level concept learning through probabilistic program induction.
\newblock {\em Science}, 350(6266):1332--1338, 2015.

\bibitem{Lake:17}
Brenden~M. Lake, Tomer~D. Ullman, Joshua~B. Tenenbaum, and Samuel~J. Gershman.
\newblock Building machines that learn and think like people.
\newblock {\em Behavioral and Brain Sciences}, 40(e253):1--72, 2017.

\bibitem{lee:03}
Tai~Sing Lee and David Mumford.
\newblock Hierarchical {B}ayesian inference in the visual cortex.
\newblock {\em Journal of Optical Society of America}, 20:1434--1448, 2003.

\bibitem{Muggleton:95}
Stephen Muggleton.
\newblock Inverse entailment and progol.
\newblock {\em New Generation Computing}, 13:245--286, 1995.

\bibitem{Muggleton:88}
Stephen Muggleton and Wray Buntine.
\newblock Machine invention of first-order predicates by inverting resolution.
\newblock In {\em Proc. 5th International Conference on Machine Learning},
  pages 339--352, 1988.

\bibitem{Nienhuys:97}
S.~H. Nienhuys-Cheng and R.~D. Wolf.
\newblock {\em Foundation of Inductive Logic Programming}.
\newblock Springer Berlin, Heidelberg, Heidelberg, 1997.

\bibitem{pearl:03}
Judea Pearl and Stuart Russell.
\newblock {\em Handbook of Brain Theory and Neural Networks}, chapter Bayesian
  Networks, pages 157--160.
\newblock MIT Press, Cambridge, Massachusetts, 2003.

\bibitem{Rao:99}
Rajesh P.~N. Rao and Dana~H. Ballard.
\newblock Predictive coding in the visual cortex: a functional interpretation
  of some extra-classical receptive-field effects.
\newblock {\em Nature Neuroscience}, 2:79--87, 1999.

\bibitem{Russell:20}
Stuart Russell and Peter Norvig.
\newblock {\em Artificial Intelligence : A Modern Approach, Fourth Edition}.
\newblock Pearson Education, Inc., London, England, 2020.

\bibitem{Tenenbaum:06}
Joshua~B. Tenenbaum, Thomas~L. Griffiths, and Charles Kemp.
\newblock Theory-based {B}ayesian models of inductive learning and reasoning.
\newblock {\em Trends in Cognitive Sciences}, 10(7):309--318, 2006.

\end{thebibliography}
%%%%%%%%%%%%%%%%%%%%%%%%%%%%%%%%%%%%%%%%%%%%%%%%%%%%%%%%%%%%%%%%%%%%%%%%%%%%%%%%%%%%%%%%%%%%%%%%%%%%%%%%%%%%%%%%%%%%%%%%%%%%%%%%%%%%%%%%%%%%%%%%%%%%%%%%%%%%%%%%%%%%%%%%%%%%%%%%%%%%%%%%%%%%%%%%%%%%%%%%%%%%%%%%%%%%%%%%%%%%%%%%%%%%%%%%%%%%%%%%%%%%%%%%%%%%%%%%%%%%%%%%%%%%%%%%%%%%%%%%%%%%%%%%%%%%%%%%%%%%%%%%
\appendix
\section{Proofs}\label{sec:appendix}
\begin{proof}[Corollary \ref{sec42:cor1}]
$p(m_{n})=0$, for all $m_{n}\in\ms{\alpha}\setminus\pms{\alpha}$. From Theorem \ref{sec42:the1}, we thus have
\begin{eqnarray*}
p_{\mu=1}(\alpha|\Delta)=\frac{\sum_{n}\ms{\alpha}_{m_{n}}\pms{\Delta}_{m_{n}}p(m_{n})}{\sum_{n}\pms{\Delta}_{m_{n}}p(m_{n})}=\frac{\sum_{n}\pms{\alpha}_{m_{n}}\pms{\Delta}_{m_{n}}p(m_{n})}{\sum_{n}\pms{\Delta}_{m_{n}}p(m_{n})}.
\end{eqnarray*}
The above equation turns out to be one if and only if $\pms{\Delta}\subseteq\pms{\alpha}$, i.e., $\Delta\ent\alpha$.
\end{proof}
\begin{proof}[Corollary \ref{sec4:cor1}]
Recall that $S\ent\alpha$ is defined as $\pms{S}\subseteq\pms{\alpha}$. $\pms{S}\subseteq\pms{\alpha}$, for all $S\in MPS(\Delta)$ iff $\bigcup_{S\in MPS(\Delta)}\pms{S}\subseteq\pms{\alpha}$, i.e., $\pams{\Delta}\subseteq\pms{\alpha}$. Since $\sum_{m_{n}\in\pams{\Delta}}\ms{\alpha}_{m_{n}}p(m_{n})=$ $\sum_{m_{n}\in\pams{\Delta}}\pms{\alpha}_{m_{n}}p(m_{n})$, Equation (\ref{sec4:eq1}) can be further expanded as follows, where the resulting value is set to one.
\begin{eqnarray}\label{sec4:eq2}
p(\alpha|\Delta)=\frac{\textstyle{\sum_{n}\pms{\alpha}_{m_{n}}\pams{\Delta}_{m_{n}}p(m_{n})}}{\textstyle{\sum_{n}\pams{\Delta}_{m_{n}}p(m_{n})}}=1
\end{eqnarray}
There is thus no model $m_{n}$ such that $m_{n}\in\pms{\alpha}\setminus\pams{\Delta}$. Therefore, $\pams{\Delta}\subseteq\pms{\alpha}$.
\end{proof}
\begin{proof}[Corollary \ref{sec4:cor2}]
Since $\pams{\Delta}=\ams{\Delta}$, Equation (\ref{sec4:eq1}) can be expanded as follows, where the resulting value is set to one.
\begin{eqnarray}\label{sec4:eq3}
p(\alpha|\Delta)=\frac{\textstyle{\sum_{n}\ms{\alpha}_{m_{n}}\ams{\Delta}_{m_{n}}p(m_{n})}}{\textstyle{\sum_{n}\ams{\Delta}_{m_{n}}p(m_{n})}}=1
\end{eqnarray}
Since $\pams{\Delta}=\ams{\Delta}$, $p(m_{n})\neq 0$, for all $m_{n}\in\ams{\Delta}$. Thus, Equation (\ref{sec4:eq3}) holds iff $\ams{\Delta}\subseteq\ms{\alpha}$, i.e., $\bigcup_{S\in MCS(\Delta)}\ms{S}\subseteq\ms{\alpha}$. This holds iff $\ms{S}\subseteq\ms{\alpha}$, for all $S\in MCS(\Delta)$.
\end{proof}
\begin{proof}[Corollary \ref{sec4:cor3}]
Since $\pms{\Delta}\neq\emptyset$, $MPS(\Delta)=\Delta$. Thus, $\pams{\Delta}=\bigcup_{S\in MPS(\Delta)}\pms{S}=\pms{\Delta}$. Therefore, Equation (\ref{sec4:eq2}) can be written as follows.
\begin{eqnarray}\label{sec4:eq4}
p(\alpha|\Delta)=\frac{\textstyle{\sum_{n}\pms{\alpha}_{m_{n}}\pms{\Delta}_{m_{n}}p(m_{n})}}{\textstyle{\sum_{n}\pms{\Delta}_{m_{n}}p(m_{n})}}=1
\end{eqnarray}
The denominator cannot be zero because of the assumption of $\pms{\Delta}\neq\emptyset$. From Equation (\ref{sec4:eq4}), $p(\alpha|\Delta)=1$ iff $\pms{\Delta}\subseteq\pms{\alpha}$.
\end{proof}
\begin{proof}[Corollary \ref{sec4:cor4}]
Since $\pms{\Delta}\neq\emptyset$, $MPS(\Delta)=\Delta$. Thus, $\pams{\Delta}=\bigcup_{S\in MPS(\Delta)}\pms{S}=\pms{\Delta}$. Namely, there is a possible model where all the elements of $\Delta$ are true. Therefore, there is a model where all the elements of $\Delta$ are true, i.e., $\ams{\Delta}=\ms{\Delta}$. Since $\pams{\Delta}=\ams{\Delta}$, we therefore have $\pams{\Delta}=\pms{\Delta}=\ams{\Delta}=\ms{\Delta}$. From Equation (\ref{sec4:eq1}), we have
\begin{eqnarray}\label{sec4:eq5}
p(\alpha|\Delta)=\frac{\textstyle{\sum_{n}\ms{\alpha}_{m_{n}}\ms{\Delta}_{m_{n}}p(m_{n})}}{\textstyle{\sum_{n}\ms{\Delta}_{m_{n}}p(m_{n})}}=1.
\end{eqnarray}
The denominator cannot be zero because $\ms{\Delta}=\pms{\Delta}\neq\emptyset$. We thus have $p(\alpha|\Delta)=1$ iff $\ms{\Delta}\subseteq\ms{\alpha}$.
\end{proof}
\begin{proof}[Corollary \ref{sec4:cor5}]
This is obvious from Corollary \ref{sec42:cor1} and Corollary \ref{sec4:cor3}.
\end{proof}
\end{document}